\documentclass[twoside]{article}
\usepackage[left=2cm,right=2cm,top=2cm,bottom=2cm]{geometry}
\usepackage{amsmath}
\usepackage{amsfonts}
\usepackage{amssymb}
\usepackage{graphicx}
\usepackage{color}
\usepackage[normalem]{ulem}
\usepackage[utf8]{inputenc}
\usepackage{amsfonts}
\usepackage{color}
\usepackage[normalem]{ulem}
\newenvironment{proof}{\noindent{\sc Proof.}}{\qed}
\newtheorem{theorem}{Theorem}[section]

\newcommand{\qed}{$\blacksquare$}

\def\RR{{\mathbb R}}

\def\SS{{\mathbb S}}

\def\x{\mathbf{x}}

\def\y{\mathbf{y}}
\def\u{\mathbf{u}}
\def\w{\mathbf{w}}

\def\v{\mathbf{v}}

\def\O{{\cal O}}

\def\C{{\mathcal C}}

\def\WW{{\mathbb W}}

\def\be{\begin{equation}}
\def\ee{\end{equation}}
\def\bea{\begin{eqnarray}}
\def\eea{\end{eqnarray}}
\def\eref#1{(\ref{#1})}

\def\donchitre#1#2{\vskip 6.5cm\noindent
\parbox[t]{1in}{\special{eps:#1.eps x=6.5cm y=5.5cm}}
\hbox to 7cm{}\parbox[t]{0.0cm}{\special{eps:#2.eps x=6.5cm y=5.5cm}}}

\def\XX{{\mathbb X}}

\title{Function approximation by deep networks}

\author{
 H.~N.~Mhaskar\thanks{
Institute of Mathematical Sciences, Claremont Graduate University, Claremont, CA 91711. The research of this author is supported in part by the Office of the Director of National Intelligence (ODNI), Intelligence Advanced Research Projects Activity (IARPA), via 2018-18032000002.
\textsf{email:} hrushikesh.mhaskar@cgu.edu} 
 \ and   T. Poggio \thanks{
   Center for Brains, Minds, and Machines, McGovern Institute for Brain Research,
   Massachusetts Institute of Technology, Cambridge, MA, 02139. The research of this author is supported by the Center for Brains, Minds and
  Machines (CBMM), funded by NSF STC award CCF-1231216.
 \textsf{tp@mit.edu} }
 }
 \date{}
\begin{document}

\maketitle
\begin{abstract}
  We show that deep networks are better than shallow networks at
  approximating functions that can be expressed as a  composition
  of functions described by a directed acyclic graph, because the deep
  networks can be designed to have the same
  compositional structure, while a shallow network cannot exploit this
  knowledge.  Thus, the blessing of compositionality
  mitigates the curse of dimensionality.  
  On the other hand, a theorem called
  good propagation of errors allows  to ``lift'' theorems about
  shallow networks to those about deep networks with an appropriate
  choice of norms, smoothness, etc.  We illustrate this in three
  contexts where each channel in the deep network calculates a
  spherical polynomial, a non-smooth ReLU network, or another zonal function
  network related closely with the ReLU network.
 
\end{abstract}
\section{Introduction}\label{intsect}
As is well known, deep networks are playing an increasingly important role in artificial intelligence, industry, and many aspects of modern life ranging from homeland security to automated cars. 
A topic of great recent interest is to examine the expressive power of deep networks to explain their remarkable success in comparison with classical shallow networks.
There are many efforts in this direction, depending upon what one defines to be the expressive power \cite{MontufarBengio2014, serra2017bounding, sharir2017expressive,telgarsky2016benefits, eldan2016power, dingxuanpap}. 

The fundamental problem of machine learning is the following. Given an integer $q\ge 1$, and data of the form $\{(\x_i,y_i)\}_{i=1}^M\subset \RR^q\times \RR$, drawn randomly from a probability distribution $\mu$,  find a model $P$ such that $P(\x_i)\approx y_i$. 
In theory, one assumes an underlying function $f$ on the unknown support of the distribution $\mu^*$ from which the $\x_i$'s are sampled, so that $y_i=f(\x_i)+\epsilon_i$, $i=1,\cdots, M$, and $\epsilon_i$ are zero mean random variables.
Equivalently, $f(\x)=\mathbb{E}_\mu(y|\x)$. 
An important aspect of the  problem of machine learning is thus viewed
as a problem of function approximation. 
A goal of this paper is to standardize the notion of expressive power in term of the ability of the network to approximate functions measured in a manner utilized in approximation theory for more than 100 years. 
Our main thesis is that the ability of deep networks to do a better approximation than shallow networks stems from their ability to mimic any compositional structure inherent in the target function; an ability that shallow networks cannot have.
On the other hand, a theorem called ``good propagation of errors'' allows us to lift results from shallow networks to those for deep networks, highlighting the importance of compositionality. 
It will be pointed out that there is no natural way to define a probability measure that can take advantage of the very important compositionality with respect to which one can define generalization error as in classical machine learning. 
In particular, the bias-variance split does not hold, and a new theory is required.
This paper summarizes some of our recent results in this direction, in particular, for deep non-smooth ReLU networks.

We will describe the central problems of approximation theory in Section~\ref{atsect} and illustrate them using the example of approximation of a function on the Euclidean (hyper-)sphere by spherical polynomials.
In Section~\ref{deepsect}, we will establish the terminology for describing deep networks. A theorem called good propagation of errors is proved and discussed in Section~\ref{propsect}. 
Applications to approximation by non-smooth ReLU networks and networks with another related activation function are discussed in Section~\ref{atsect}. 
The relationship of our results with some others in the literature is discussed in Section~\ref{relatedsect}.

\section{Basic ideas in approximation theory}\label{atsect}
A central problem in approximation theory is to investigate the quality of approximation of an unknown function given finite amount of information about the function. 
In order to do so, one assumes that the target function $f$ is in some Banach space $\XX$ with norm $\|\cdot\|$. 
The function needs to be approximated by models coming from a nested sequence of sets $V_0\subset\cdots\subset V_n\subset V_{n+1}\subset \cdots$ so that $\displaystyle\cup_{n=0}^\infty V_n$ is dense in $\XX$. 
One of the most important quantities in approximation theory is the \textit{degree of approximation}, defined by
\begin{equation}
\label{degapprox}
\mathsf{dist}(\XX;f,V_n)=\inf_{P\in V_n}\|f-P\|.
\end{equation}
The assumption that $\displaystyle\cup_{n=0}^\infty V_n$ is dense in $\XX$ means that $\displaystyle\lim_{n\to\infty}d_n(\XX;f,V_n)=0$. 
The rate of this convergence clearly depends upon further assumptions  on $f$, called \textit{prior} in machine learning parlance, and    \textit{smoothness class} in approximation theory. 
Typically, this class is defined in terms of a \textit{smoothness parameter} $\gamma$ as a subspace $\WW_\gamma$ of $\XX$.

Constructing the minimizer in (\ref{degapprox}) is generally not of any interest. 
Such a minimizer can be hard to obtain computationally, and does not have many desirable properties; e.g., it is generally not sensitive to the local properties of $f$.
Instead, the central themes of approximation theory are:
\begin{description}

\item[\textbf{Direct theorem}]  This states that if $f\in \WW_\gamma$,  
\begin{equation}\label{genericdirecttheo}
\mathsf{dist}(\XX;f,V_n)=\O(n^{-s})
\end{equation}  
for some $s$ depending upon $\gamma$ and other parameters, e.g., the number of input variables to $f$. 
\item[\textbf{Construction, aka training}] Give a method to construct $P\in V_n$ from the given information on $f$ such that $\|f-P\|=\O(n^{-s})$, and study the connection between the amount of information available and $n$ for which such a construction is possible.
 \item[\textbf{Width theorem}]This states that if we can only assume that $f\in K \subset \WW_\gamma$ for a compact subset $K$, and $n$ pieces of information are allowed on $f$ (in the form of a continuous mapping $S:K\to\RR^n$), then no matter how one constructs an approximation to $f$ from this information, i.e., $A(S(f))\in\XX$, the worst case error under the assumption that $f\in K$ is $\Omega(n^{-s})$. 
This asserts merely the existence of $f\in K$ for which the lower estimate holds.
\item[\textbf{Converse theorem}]  This states that the estimate (\ref{genericdirecttheo}) implies that $f\in \WW_\gamma$. 
\emph{This is a statement about individual functions, not about the whole class of functions.} 
Also, while the width estimate involves only continuous parameter selection, a converse theorem does not stipulate this.
\end{description}

We discuss an example in connection with approximation on a Euclidean sphere of $\RR^{q+1}$ for some integer $q\ge 1$: 
$$
\SS^q=\{\x=(x_1,\cdots, x_{q+1})\in\RR^{q+1}: x_1^2+\cdots+x_{q+1}^2=1\}.
$$
We will be interested in approximating continuous functions on $\SS^q$, so that the Banach space is $C(\SS^q)$ equipped with the uniform norm $\|\cdot\|_{\SS^q}$. 
The restriction of an algebraic polynomial in $q+1$ real variables of total degree $n$ to $\SS^q$ is called a \textit{spherical polynomial} of degree $n$. 
The space of all spherical polynomials of degree $<n$ is denoted by $\Pi_n^q$. 
Thus, $V_n=\Pi_n^q$. 
We will denote $\mathsf{dist}(C(\SS^q);f,\Pi_n^q)$ by $E_{q;n}(f)$.

The smoothness class is defined as follows.
If $\Delta$ is the negative Laplace-Beltrami operator on $\SS^q$,  
a $K$-functional on the space $C(\SS^q)$  is defined by
\begin{equation}
\label{kfuncdef}
K_r(f,\delta)=\inf\{\|f-g\|_{\SS^q}+
\delta^r\|(I+\Delta)^{r/2}g\|_{\SS^q}\}, \qquad \delta>0,
\end{equation}
where $r$ is an even integer, and the infimum is taken over all $g$ for which  $(I+\Delta)^{r/2}g\in C(\SS^q)$. 
The class $W_{q;\gamma}$ is defined by
\begin{equation}
\label{smoothness_class_def}
W_{q;\gamma}=\left\{ f\in C(\SS^q) : \|f\|_{W_{q;\gamma}}=\|f\|_{\SS^q}+\sup_{\delta\in (0,1)}\delta^{-\gamma}K_r(f,\delta) <\infty\right\}
\end{equation}
for an even integer $r>2\gamma$. The following estimate (\ref{equivtheorem}) that the class $W_{q;\gamma}$ (although not the norm $\|f\|_{W_{q;\gamma}}$) is independent of the choice of $r$.

It is proved in \cite{pawelke1972, lizorkin1994nikol} that there exist positive constants $c_1,c_2$ depending only on $q, \gamma, r$ such that
\begin{equation}
\label{equivtheorem}
c_1\|f\|_{W_{q;\gamma}}\le \|f\|_{\SS^q}+\sup_{n\ge 1}n^{\gamma} E_{q;n}(f)\le c_2\|f\|_{W_{q;\gamma}}.
\end{equation}
The second inequality gives an estimate on the degree of approximation in terms of the smoothness class, and represents the direct theorem. 
The first inequality asserts that the rate at which the degree of approximation converges to $0$ determines the smoothness class to which the target function belongs; i.e., a converse theorem.
The converse theorem in particular is stronger than the width theorem.

A construction of a polynomial approximation that yields the bounds is given in \cite{lizorkin1994nikol} in the case when spectral information is available, and in \cite{quadconst} in the case when noisy values of the function are given at arbitrary points on the sphere.

We note that the dimension of $\Pi_n^q\sim n^q$. 
Therefore, in terms of the number of parameters $M$ involved in the approximation, the rate in (\ref{equivtheorem}) is $\sim M^{-\gamma/q}$. This exponential dependence on $q$ is called \textit{curse of dimensionality}; the quantity $q/\gamma$ is called the \textit{effective dimension} of $W_{q;\gamma}$.

\section{Deep networks and compositional functions}\label{deepsect}
A commonly used definition of a deep network is the following. 
Let $\phi :\RR\to\RR$ be an activation function; applied to a vector $\x=(x_1,\cdots,x_q)$, $\phi(\x)=(\phi(x_1),\cdots,\phi(x_q))$. Let $L\ge 2$ be an integer, for $\ell=0,\cdots,L$, let $q_\ell\ge 1$  be an integer ($q_0=q$),  $T_\ell :\RR^{q_\ell}\to \RR^{q_{\ell+1}}$ be an affine transform, where $q_{L+1}=1$.  A deep network with $L-1$ hidden layers is defined as the compositional function
$$
\x\mapsto T_L(\phi(T_{L-1}(\phi(T_{L-2}\cdots\phi(T_0(\x))\cdots).
$$
There are several shortcomings for this definition. 
First, a function may have more than one compositional representation, so that the affine transforms and $L$ are not determined uniquely by the function itself. 
Second, this notion does not capture the connection between the nature of the target function and its approximation.
Third, the affine transforms $T_\ell$ define a special directed acyclic graph (DAG). 
It is difficult to describe notions of weight sharing, convolutions, sparsity, etc. in terms of these transforms.

Therefore, we follow \cite{dingxuanpap} and fix a DAG to represent both the target function and its approximation. 
Let $\mathcal{G}$ be a DAG, with the set of nodes $V\cup \mathbf{S}$, where $\mathbf{S}$ is the set of source nodes, and $V$ that of non-source nodes. 
We assume that there is only one sink node, $v^*$.
A \textit{$\mathcal{G}$-function} is defined as follows. 
The in-edges to each node in $V$ represents an input real variable. 
For each node $v\in V\cup \mathbf{S}$, we denote its in-degree by $d(v)$. A node $v\in V\cup\mathbf{S}$ itself represents the evaluation of a real valued function $f_v$ of the $d(v)$ inputs. The out-edges fan out the result of this evaluation. 
Each of the source node obtains an input from some Euclidean space. 
Other nodes can also obtain such an input, but by introducing dummy nodes, it is convenient to assume that only the source nodes obtain an input from the Euclidean space. 
In summary, a $\mathcal{G}$-function is actually a set of functions $\{f_v : v\in V\cup\mathbf{S}\}$, each of which will be called a \textit{constituent function}.

For example, the DAG $\mathcal{G}$ in Figure~\ref{graphpict} (\cite{dingxuanpap}) represents the compositional function
\begin{eqnarray}
\label{gfuncexample}
f^*(x_1,\cdots, x_9)&=&
h_{19}(h_{17}(h_{13}(h_{10}(x_1,x_2,x_3, h_{16}(h_{12}(x_6,x_7,x_8,x_9))), h_{11}(x_4,x_5)), \nonumber\\ 
&& \qquad\qquad h_{14}(h_{10},h_{11}), h_{16}), h_{18}(h_{15}(h_{11},h_{12}),h_{16})).
\end{eqnarray}
The $\mathcal{G}$-function is $\{h_{10},\cdots,h_{19}= f^*\}$; the source nodes $\mathbf{S}=\{ h_{10}, h_{11}, h_{12}\}$, $V=\{h_{13},\cdots, h_{19}\}$.

\begin{figure}[h]
\begin{center}
\includegraphics[width=4in, height=3in]{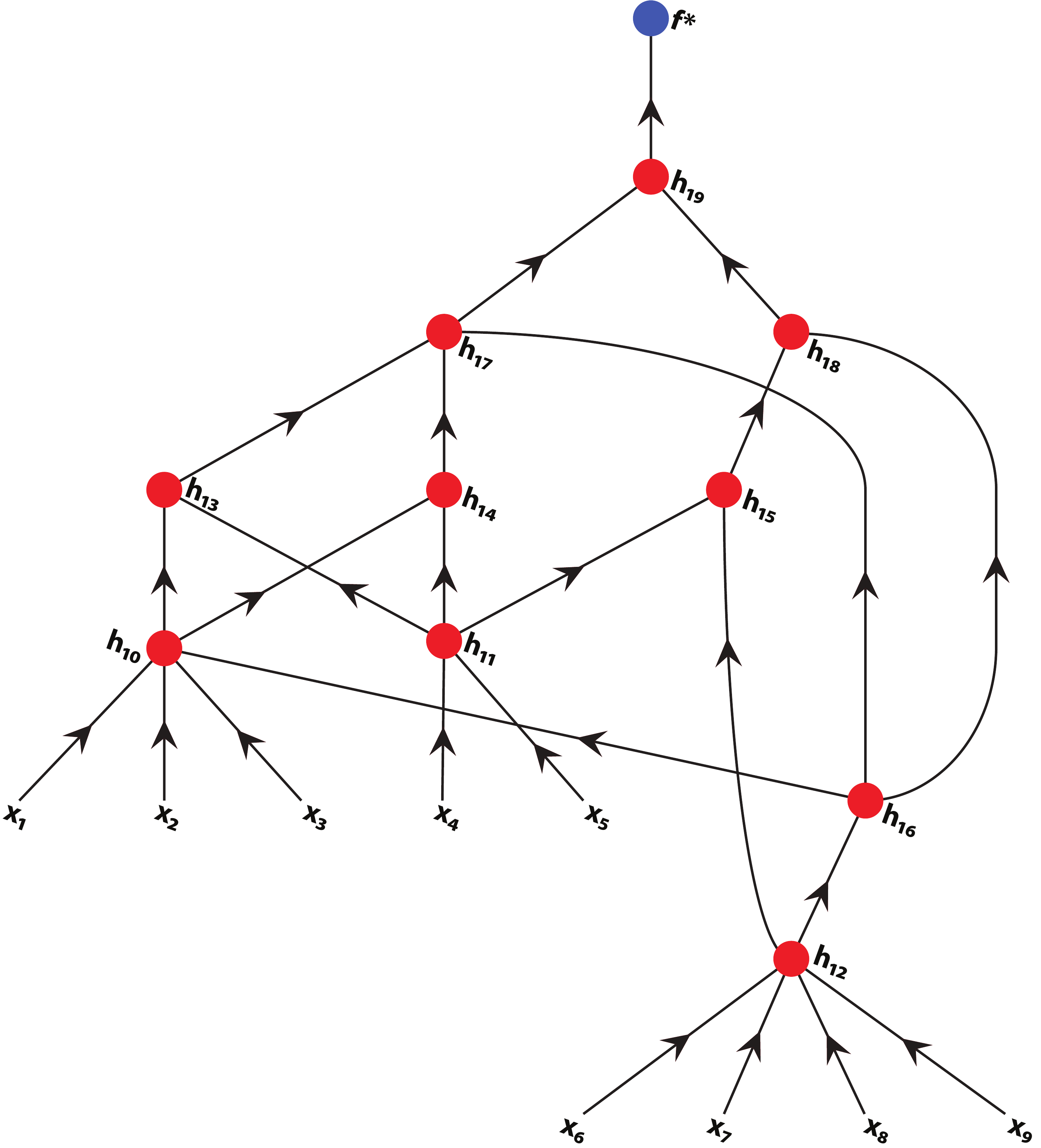} 
\end{center}

\caption{This figure from \cite{dingxuanpap} shows an example of a $\mathcal{G}$--function ($f^*$ given in (\ref{gfuncexample})). The vertices $V\cup \mathbf{S}$ of the DAG $\mathcal{G}$ are denoted by red dots. The black dots represent the inputs; the input to the various nodes as indicated by the in--edges of the red nodes. The blue dot indicates the output value of the $\mathcal{G}$--function, $f^*$ in this example.}
\label{graphpict}
\end{figure}

If $v\in \mathbf{S}$,  the (vector of) \textit{variables seen by $v$} are those which are input to $v$. 
For other $v\in V$, the \textit{variables seen by $v$} are defined recursively as the vector of variables obtained by concatenating the variables seen by each of the children of $v$ in order. 
In particular, there is a notation overload. The function $f_v$ is a function of $d(v)$ variables input to the vertex $v$. It is also a function of the variables seen by $v$.
For example, in the DAG of Figure~\ref{graphpict}, $h_{11}$ sees the variables $(x_4, x_5)$, $h_{13}$ is a function of two variables, namely, the outputs of $h_{10}$ and $h_{11}$, but it is also a function of the variables $(x_1,\cdots,x_5)$ which are seen by $h_{13}$.
We will explain what meaning is intended if we find it warranted.

In the remainder of this paper, we will assume $\mathcal{G}$ to be a fixed DAG.
%

\section{Good propagation of errors}\label{propsect}
The following Theorem~\ref{goodproptheo} is the main technical tool that allows us to reduce the problem of approximation by deep networks  to a series of approximations by shallow networks.
In this theorem, for integer $d\ge 1$, let $\rho_d$ be a metric on $\RR^d$. 
\begin{theorem}
\label{goodproptheo}
 Let $\{f_v\}$ be a $\mathcal{G}$-function satisfying the following Lipschitz condition: there exists a constant $L>0$ such that for $(x_1,\cdots,x_{d_v}), (y_1,\cdots,y_{d_v})\in\RR^{d(v)}$,
\begin{equation}
\label{liptschitz}
|f_v(x_1,\cdots,x_{d_v})-f_v(y_1,\cdots,y_{d_v})|\le L\rho_{d(v)}((x_1,\cdots,x_{d_v}), (y_1,\cdots,y_{d_v})).
\end{equation}
Let $\{g_v\}$ be a $\mathcal{G}$-function.
Let $w\in V$,  $\{u_1,\cdots,u_s\}\subset V$ be the children of $w$, and $\x_{u_1},\cdots, \x_{u_s}$ be the variables seen by $u_1,\cdots, u_s$ respectively.
Then
\begin{eqnarray}
\label{good_prop_err}
\lefteqn{|f_v(\x_{u_1},\cdots, \x_{u_s})-g_v(\x_{u_1},\cdots, \x_{u_s})|}\nonumber\\
&\!\!=\!\!&|f_v(f_{u_1}(\x_{u_1}),\cdots,f_{u_s}(\x_{u_s}))-g_v(g_{u_1}(\x_{u_1}),\cdots,g_{u_s}(\x_{u_s}))|\nonumber\\
&\!\!\le\!\!& \sup_{\y\in\RR^{d(v)}}|f_v(\y)-g_v(\y)| + L\rho_{d(v)}((f_{u_1}(\x_{u_1}),\cdots,f_{u_s}(\x_{u_s})),(g_{u_1}(\x_{u_1}),\cdots,g_{u_s}(\x_{u_s}) ).
\end{eqnarray}

\end{theorem}

\begin{proof}\ 
By triangle inequality followed by (\ref{liptschitz}), we get
\begin{eqnarray*}
\lefteqn{|f_v(f_{u_1}(\x_{u_1}),\cdots,f_{u_s}(\x_{u_s}))-g_v(g_{u_1}(\x_{u_1}),\cdots,g_{u_s}(\x_{u_s}))|}\\
&\le& |f_v(g_{u_1}(\x_{u_1}),\cdots,g_{u_s}(\x_{u_s}))-g_v(g_{u_1}(\x_{u_1}),\cdots,g_{u_s}(\x_{u_s}))|\\
&&\qquad + |f_v(f_{u_1}(\x_{u_1}),\cdots,f_{u_s}(\x_{u_s}))-f_v(g_{u_1}(\x_{u_1}),\cdots,g_{u_s}(\x_{u_s}))|\\
&\le& \sup_{\y\in\RR^{d(v)}}|f_v(\y)-g_v(\y)| + L\rho_{d(v)}((f_{u_1}(\x_{u_1}),\cdots,f_{u_s}(\x_{u_s})),(g_{u_1}(\x_{u_1}),\cdots,g_{u_s}(\x_{u_s}) )).
\end{eqnarray*}
\end{proof}

We illustrate Theorem~\ref{goodproptheo} using the example of approximation by spherical polynomials as in Section~\ref{atsect}. 
We note first  that the transformation
\begin{equation}
\label{sphere_rr_transform}
(x_1,\cdots,x_d)\mapsto \left(\frac{x_1}{\sqrt{|\x|^2+1}}, \cdots, \frac{x_d}{\sqrt{|\x|^2+1}}, \frac{1}{\sqrt{|\x|^2+1}}\right)
\end{equation}
is a one-to-one correspondence between $\RR^d$ and the open upper hemisphere $\SS^d_+$. 
For a function $f:\RR^d\to\RR$ vanishing at infinity, one can therefore associate in a one-to-one manner an even function on $\SS^d$ which shares all the smoothness properties of $f$. 
In the notation of Theorem~\ref{goodproptheo}, if we assume that all the $\mathcal{G}$-functions involved are continuous, the points such as $(f_{u_1}(\x_{u_1}),\cdots,f_{u_s}(\x_{u_s}))$ may thus be thought of as points on a compact subset of $\SS^s_+$.
Therefore, with some simple modifications, we may assume that the inputs to all the constituent functions are from the appropriate spheres.
Moreover, restricted to compact subsets of $\RR^d$, the usual Euclidean metric on $\RR^d$ is equivalent to the  metric $\rho_d$ on $\RR^d$  induced by the geodesic distance $\varrho_d$ on $\SS^d$. Therefore, we may write (\ref{good_prop_err}) in the form
\begin{equation}
\label{sph_good_prop_err}
|f_v(\x_{u_1},\cdots, \x_{u_s})-g_v(\x_{u_1},\cdots, \x_{u_s})|\le \|f_v-g_v\|_{\SS^{d(v)}}+L\sum_{k=1}^{d(v)}\|f_{u_k}-g_{u_k}\|_{\SS^{d(u_k)}}.
\end{equation}

Motivated by Theorem~\ref{goodproptheo}, we define the following notion. 
Let $W_d$ be a class of functions of $d$ variables with norm (or semi-norm) $\|\cdot\|_{W_d}$. 
The class $\mathcal{G}W$ consists of all $\mathcal{G}$-functions $\{f_v\}$ such that each $f_v\in W_{d(v)}$. 
We define
\begin{equation}
\label{gfunctnorm}
\|\{f_v\}\|_{\mathcal{G}W}=\sum_{v\in V}\|f_v\|_{W_{d(v)}};
\end{equation}
i.e., we use the tensor product norm on $\prod_{v\in V}W_{d(v)}$. 
For example, $\mathcal{G}\Pi_n$ is the class of all $\mathcal{G}$-functions of the form $\{P_v \in\Pi_n^{d(v)}\}$, 
$$
\|\{f_v\}\|_{\mathcal{G}W_\gamma}=\sum_{v\in V}\|f_v\|_{W_{d(v);\gamma}}, \quad E_n(\mathcal{G},\{f_v\})=\sum_{v\in V}E_{d(v);n}(f_v).
$$
We note that the fact that $E_n(\mathcal{G},\{f_v\})=\O(n^{-\gamma})$ is equivalent to the fact that $E_{d(v);n}(f_v)=\O(n^{-\gamma})$ for each $v\in V$.
Together with (\ref{equivtheorem}), Theorem~\ref{goodproptheo} leads to the following
\begin{theorem}
\label{deep_sph_poly_theo}
Let $\{f_v\}$ be a $\mathcal{G}$-function such that (\ref{liptschitz}) is satisfied with $\rho_d$ induced by the geodesic metric on $\SS^d$. 
Then there exist positive constants $c_3, c_4$ independent of the functions $\{f_v\}$ or $n$ such that
\begin{equation}
\label{deep_sph_poly_est}
c_3\|\{f_v\}\|_{\mathcal{G}W_\gamma}\le \sum_{v\in V}\|f_v\|_{\SS^{d(v)}}+\sup_{n\ge 1}n^{\gamma} E_n(\mathcal{G},\{f_v\})\le c_4\|\{f_v\}\|_{\mathcal{G}W_\gamma}.
\end{equation}
\end{theorem}

We end this section by pointing out another important feature of Theorem~\ref{goodproptheo}.
It is customary in machine learning to measure the generalization error between a function and its approximation using an appropriate $L^2$ norm. 
In (\ref{good_prop_err}), the argument of $f_v$ is different (and in particular, differently distributed) from that of $g_v$. 
Thus, there is no natural measure with respect to which one can take the $L^2$ norm while preserving the advantages of compositionality.
Therefore, in the theory of function approximation by deep networks, one has to use the uniform norm.
In turn, this means that the usual bias-variance split does not work anymore, and one has to develop an entirely new paradigm.

\section{Approximation by ReLU networks}\label{relusect}
A ReLU network has the form $\x\mapsto\sum_{k=1}^N a_k(\x\cdot \y_k+b_k)_+$. Since $|t|=t_++(-t)_+$, $t_+=(|t|+t)/2$, we find it convenient to study instead networks of the form $\x\mapsto\sum_{k=1}^N a_k|\x\cdot \y_k+b_k|$. 
Writing $\w_k=(|\y_k|^2+b_k^2)^{-1/2}(\y_k,b)$ and recalling the transformation between $\RR^q$ and $\SS^q$, the problem of approximation of functions on $\RR^q$ by networks of this form is equivalent to that of approximation of functions on $\SS^q$ by zonal function networks of the form
$\x\mapsto\sum_{k=1}^N a_k|\x\cdot \w_k|$. 

Next, we define a smoothness class for approximation by such networks  \cite{sphrelu, dingxuanpap}.
In this section, we denote the dimension of the space of the restrictions to the sphere of all homogeneous harmonic polynomials of degree $\ell$ by $d_\ell^q$, $\ell=0,1,\cdots$, and the set of orthonormalized spherical harmonics on $\SS^q$ by $\{Y_{\ell,k}\}_{k=1}^{d_\ell^q}$. 
we recall the addition formula
\be\label{additionformula}
\sum_{k=1}^{d_\ell}Y_{\ell,k}(\u)\overline{Y_{\ell, k}(\v)}=\omega_{q-1}^{-1}p_\ell(1)p_\ell(\u\cdot\v),
\ee
where $p_\ell$ is the degree $\ell$ ultraspherical polynomial with positive leading coefficient, with the set $\{p_\ell\}$ satisfying
\be\label{orthonormal}
\int_{-1}^1 p_\ell(t)p_j(t)(1-t^2)^{q/2-1}dt =\delta_{j,\ell}, \qquad j, \ell =0, 1,\cdots.
\ee
The function $t\to |t|$ can be expressed in an expansion
\be\label{absseries}
|t|\sim p_0-\sum_{\ell=1}^\infty \frac{\ell-1}{\ell(2\ell-1)(\ell+q/2)}p_{2\ell}(0)p_{2\ell}(t), \qquad t\in [-1,1],
\ee
with the series converging on compact subsets of $(-1,1)$. 

If $f\in C(\SS^q)$, then we define
\be\label{fourcoeff}
\hat{f}(\ell, k)=\int_{\SS^q} f(\u)Y_{\ell,k}(\u)d\mu^*(\u).
\ee
We note that if $f$ is an even function, then  $\hat{f}(2\ell+1, k)=0$ for $\ell=0,1,\cdots$. 
In this context, the place of the operator $(I+\Delta)^{1/2}$ is taken by the operator $\mathcal{D}_{q;|\cdot|}$ defined formally by
\be\label{formderdef}
\widehat{\mathcal{D}_{q;|\cdot|} f}(2\ell,k)=\left\{\begin{array}{ll}
\hat{F}(0,0), & \mbox{if  $\ell=0$},\\[1ex]
\displaystyle -\frac{\ell(2\ell-1)(\ell+q/2)p_{2\ell}(1)}{\omega_{q-1}(\ell-1)p_{2\ell}(0)}\hat{F}(2\ell, k), &\mbox{ if $\ell=1,2,\cdots$,}
\end{array}\right.
\ee
and $\widehat{\mathcal{D}_{q;|\cdot|} f}(2\ell+1,k)=0$ otherwise.  
The space of all $f\in C(\SS^q)$ for which $\mathcal{D}_{q;|\cdot|}f \in C(\SS^q)$ is denoted by $Y_q$. 
We set
$$
\|f\|_{Y_q}=\|f\|_{C(\SS^q)}+\|\mathcal{D}_{q;|\cdot|}f \|_{\SS^q}.
$$
It is proved in \cite{sphrelu} that if $f\in Y_q$, then there exists a network of the form 
\be\label{sphrelurep}
G(\x)=\sum_{k=1}^N a_k|\x\cdot\w_k|
\ee
 such that
\be\label{sphreluapprox}
\|f-G\|_{\SS^q}\le \frac{c_4}{N^{2/q}}\|f\|_{Y_q}.
\ee
The class of all networks of the form $G$ is denoted $R_{q;N}$.
Our result in \cite{sphrelu} is in fact a constructive result. 
Thus, we work with data of the form $\{(\x_j,f(\x_j))\}_{j=1}^M$, $\x_j\in \SS^q$. 
If the points $\{\x_j\}$ are sufficiently dense on $\SS^q$, then we have shown that a network $G$ of the form \eref{sphrelurep} can be constructed with $N\sim M$, the coefficients $a_k$ can be chosen to be linear combinations of $\{f(\x_j)$'s with weights independent of $f$, and the points $\w_k$ can be chosen independently of the data. 
Thus, there is no training involved in the classical sense.

Theorem~\ref{goodproptheo} allows to ``lift'' this upper bound to the following corresponding bound for deep ReLU networks.
\begin{theorem}
\label{deep_relu_theo}
Let $\{f_v\}$ be a $\mathcal{G}$-function such that each $f_v$ satisfies (\ref{liptschitz})  with $\rho_{d(v)}$ induced by the geodesic metric on $\SS^{d(v)}$. In addition, let each $f_v\in Y_{d(v)}$. Let $d_\mathcal{G}=\max_{v\in V}d(v)$. Then there exists a deep network in $\mathcal{G}R_N$; i.e., a $\mathcal{G}$-function $\{g_v\}$ such that every $g_v\in R_{d(v);N}$ such that
\be\label{deep_sphrelu_approx}
\sum_{v\in V}\|f_v-g_v\|_{\SS^{d(v)}}\le \frac{c_5}{N^{2/d_\mathcal{G}}}\|\{f_v\}\|_{\mathcal{G}Y}.
\ee
\end{theorem}
For example, if $\mathcal{G}$ is a binary tree with $1024$ leaves, then a shallow network as in (\ref{sphreluapprox}) with $N$ neurons yields a degree of approximation $O(N^{-1/(512)})$, while a deep network as in (\ref{deep_sphrelu_approx}) yields a degree of approximation $O(N^{-1})$; a substantial improvement.

The ``derivative'' $\mathcal{D}_{|\cdot|}$ is very unusual in that instead of being a local function, it is supported on equators perpendicular to the point in question.
This is illustrated by Figure~\ref{reluapproxfig} from \cite{sphrelu}.
\begin{figure}[h]
\begin{center}
\begin{minipage}{0.4\textwidth}
\begin{center}
\includegraphics[width=0.5\textwidth,keepaspectratio]{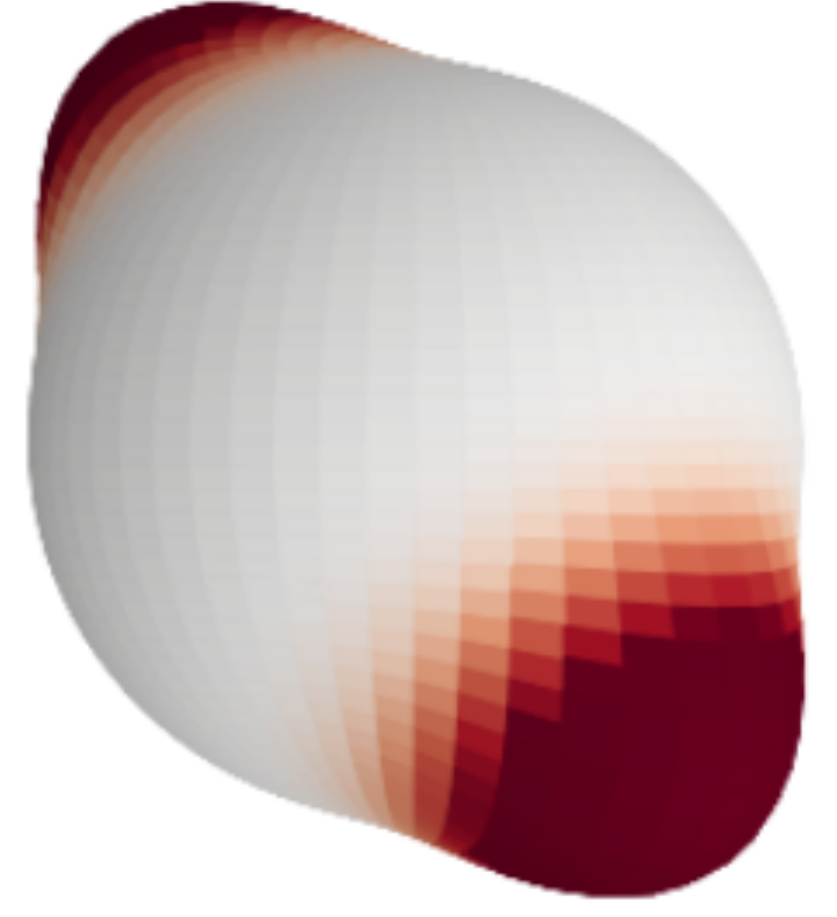} 
\end{center}
\end{minipage}
\begin{minipage}{0.4\textwidth}
\begin{center}
\includegraphics[width=0.5\textwidth,keepaspectratio]{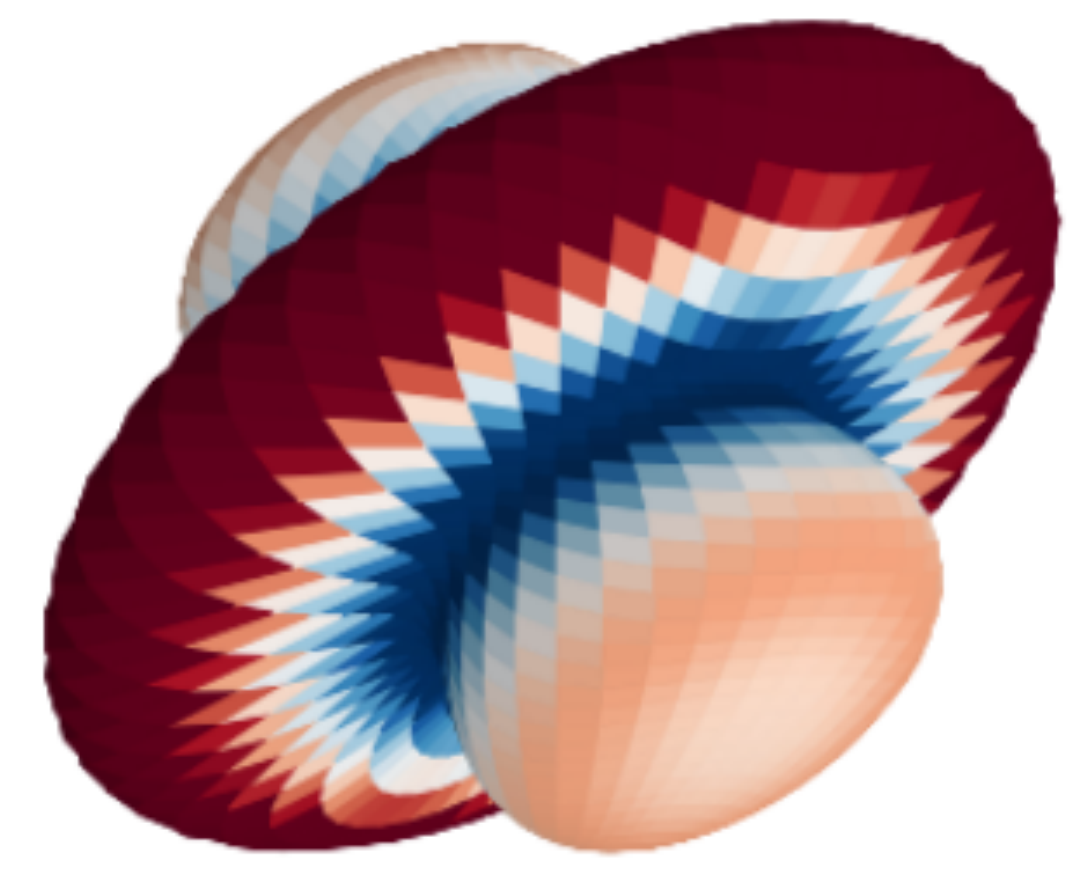} 
\end{center}
\end{minipage}
\end{center}
\caption{On the left, with $\x_0=(1,1,1)/\sqrt{3}$, the graph of $f(\x)=[(\x\cdot\x_0-0.1)_+]^8 + [(-\x\cdot\x_0-0.1)_+]^8$. On the right, the graph of $\mathcal{D}_{\phi_\gamma}(f)$. Courtesy: D. Batenkov.}
\label{reluapproxfig}
\end{figure}

Omitting the requirement that the mapping $f\mapsto (a_1,\cdots,a_N, \w_1,\cdots, \w_N)$ be continuous, we have proved in \cite{mhaskar2019dimension} that the estimate in \eref{sphreluapprox} can be improved to $\O(N^{(q+3)/(2q)})$. 
Of course, the bounds in \eref{deep_sphrelu_approx} are also improved accordingly. 
For example, if $q=1024$, and DAG structure is a full binary tree, then the improvement in the estimate for deep network is only (up to a logarithmic term) $\O(N^{-1.25})$, while the same for a shallow network is $\O(N^{-0.5015})$. 
With the requirement about the network being trained with samples of $f$ (i.e., a continuous parameter selection), the improvement is (up to a logarithmic term) $\O(N^{-1})$ for deep networks, over $\O(N^{-0.002})$ for shallow networks. 
Since a converse theorem does not stipulate continuous parameter selection, a converse theorem is not possible in this context.
However, we conjecture that a width theorem is true.

In contrast to the ReLU networks, if we consider the spherical convolution function 
\be\label{sph_conv_relu}
\phi(\x\cdot \y)=\int_{\SS^q}|\x\cdot\u||\u\cdot\y|d\mu^*(\u),
\ee
then a complete theory emerges by combining the results in \cite{eignet} with Theorem~\ref{goodproptheo}. 
An interesting feature of this theory is that the complexity of the network is not measured in terms of the number of neurons but the minimal separation among the neurons.
If $\C\subset \SS^q$ is a finite subset, we define the minimal separation $\eta(\C)$ and mesh norm $\delta(\C)$ of $\C$ by
\be\label{minsepdef}
\eta(\C)=\min_{\x,\y\in\C, \x\not=\y}\varrho_q(\x,\y),\quad \delta(\C)=\max_{\x\in\SS^q}\min_{\y\in \C}\varrho_q(\x,\y),
\ee
where $\varrho_q$ is the geodesic distance on $\SS^q$. 
By replacing $\C$ by a suitable subset, we may assume that
\be\label{uniformity}
\delta(\C)\le 2\eta(\C)\le 4\delta(\C).
\ee
For a finite subset $\C\subset \SS^q$, the set $\mathcal{N}(q;\C)$ comprises networks of the form
$\x\mapsto\sum_{\y\in\C}a_\y\phi(\x\cdot\y)$. 
We note that the number of neurons in a network in $\mathcal{N}(q;\C)$ is $O(\eta(\C)^{-q})$, but given $N$, it easy to construct $\C$ with $N$ elements for which $\eta(\C)^{-q}\gg N$. 

Omitting many nuances and using a different notation, \cite[Theorem~3.3]{eignet} (applied to the sphere) can be restated in the following form. 
\begin{theorem}
\label{shallow_eignet_theo}
Let $0<\gamma<3$ and $f\in W_{q;\gamma}$. 
For any set $\C$ satisfying (\ref{uniformity}), there exists $G\in \mathcal{N}(q;\C)$ such that
\be\label{shallow_direct}
\|f-G\|_{\SS^q}\le c_6\eta(\C)^\gamma \|f\|_{W_{q;\gamma}}.
\ee
Conversely, let $\C_m$ be a nested sequence of sets satisfying (\ref{uniformity}), and for each integer $m\ge 1$, $\eta(\C_m)\ge 1/m$. If $f\in C(\SS^q)$ and
$\mathsf{dist}(C(\SS^q);f,\mathcal{N}(q;\C_m))=O(m^{-\gamma})$, then $f\in W_{q;\gamma}$.
\end{theorem}
Using Theorem~\ref{goodproptheo}, this theorem can be lifted as before to the following theorem for deep networks.
\begin{theorem}
\label{deep_eignet_theo}
{\rm (a)} For each $v\in V$, let $\C_v\subset \SS^q$ be finite subsets satisfying (\ref{uniformity}). Let $\eta=\max\eta(\C_v)$. Let $0<\gamma<3$, and $\{f_v\}\in \mathcal{G}W_\gamma$. In addition, we assume that each  $f_v$ satisfies (\ref{liptschitz})  with $\rho_{d(v)}$ induced by the geodesic metric on $\SS^{d(v)}$. Then there exists a $\mathcal{G}$-function $\{G_v\}$ such that each $G_v\in \mathcal{N}(d(v);\C_v)$ and
\be\label{deep_eignet_direct}
\sum_{v\in V}\|f_v-G_v\|_{\SS^{d(v)}}\le c_7\eta^\gamma \sum_{v\in V}\|f_v\|_{d(v);\gamma}.
\ee
{\rm (b)} Conversely, for each $v\in V$, let $\C_{m,v}$ be a nested sequence of finite subsets of $\SS^{d(v)}$ satisfying (\ref{uniformity}) and $\eta(\C_{m,v})\ge 1/m$. If $\{f_v\}$ is a $\mathcal{G}$-function such that each $f_v\in C(\SS^{d(v)})$ and there exists a sequence of $\mathcal{G}$-functions $\{G_{m,v}\}$ such that each $G_{m,v}\in \mathcal{N}(d(v),\C_{m,v})$ and 
$$
\sum_{v\in V}\|f_v-G_{m,v}\|_{\SS^{d(v)}} =O(m^{-\gamma}),
$$
then each $f_v\in W_{d(v);\gamma}$.
\end{theorem}

\section{Related works}\label{relatedsect}
There is a deluge of papers on the expressive power of deep networks and their superiority over shallow networks. 
We cite a few of these.
The papers \cite{MontufarBengio2014, serra2017bounding}  measure the expressive power by the number of linear pieces into which the network partitions the domain space. 
This measurement overlooks the fact that the optimal number of pieces ought to depend upon the function being approximated.
It is shown in \cite{sharir2017expressive} that deep networks are better when the complexity is measured in terms of the rank of certain tensors.
It is not clear how this criterion relates to the problem of function approximation.
The papers \cite{telgarsky2016benefits, eldan2016power} establish the existence of functions which cannot be approximated well by neural networks with a given graph structure. 
This anticipates the compositionality of the networks being represented by a DAG structure, but does not address the compositional nature of the target function itself.
The papers \cite{safran2017depth, chui1996limitations, chui_zhou} show that specific functions such as the characteristic functions of balls and  radial functions cannot be approximated well by shallow ReLU networks.
In \cite{multilayer}, it is shown that by using the function $t\mapsto (t_+)^2$ as the activation function, one can synthesize any spline or polynomial exactly with a network with sufficient depth. 
In particular, one can synthesize any given partition of the Euclidean space into linear regions arbitrarily closely.
In \cite{hanin2017universal} estimates on the degree of uniform approximation are given in terms of the modulus of continuity, where the number of neurons in each layer is fixed at $2q+1$, but the number of layers is inversely proportional to the modulus of continuity and fixed width.
The paper \cite{bach2014} obtains bounds on the degree of approximation of Lipschitz continuous functions by ReLU networks. 
The idea of transforming the problem from the Euclidean space to that on the sphere is used in this paper as well.
This paper also considers approximation by spherical convolutions as in (\ref{sph_conv_relu}). 
Our estimates are under different assumptions, and are better.
Lower bounds for universal approximation of Lipschitz functions by ReLU networks are given in \cite{yarotsky2017error, yarotsky2018optimal}, and for twice differentiable functions in \cite{SafranShamir2016}.
In particular, \cite{yarotsky2018optimal} gives a detailed analysis, showing the order of magnitude of the degree of approximation of Lipschitz continuous functions cannot be better than $N^{-2/q}$, where $N$ is the number of neurons. 
The bound (\ref{sphreluapprox}) clearly achieves this as an upper bound, but with a different class of functions.
We conjecture that the class of functions introduced in this paper is the best possible, in the sense that the estimate (\ref{sphreluapprox}) cannot be improved in terms of nonlinear widths.
However, a converse theorem is probably not true.
Finally, we note that explicit expressions for the kernels $\phi$ defined in (\ref{sph_conv_relu}) are easy to deduce from those given in \cite{cho2009kernel} where the function $t\mapsto \max(t,0)$ is used in place of $|\circ|$.

\section{Conclusions}\label{concludesect}
We have demonstrated several concepts in this paper.
First, we have shown that deep networks have a better approximation power than shallow networks because they are capable of reflecting any compositional structure in the target function, while shallow networks cannot.
Second, we have pointed out an important tool in this theory called good propagation of errors which enables us to lift theorems on approximation power of shallow networks to those of deep networks if all the constituent functions are Lipschitz continuous.
Third, we have argued that in order to use this tool, there is no natural measure at each step with respect to which the error can be measured in the $L^2$-norm as customary in machine learning.
In particular, the usual bias-variance split does not work anymore, and a new paradigm is necessary.
Fourth, we obtained converse theorems for approximation by certain kernels obtained from the ReLU functions which enable us to verify from the observed degree of approximation the prior smoothness condition which the target function must satisfy.

We note that the question of whether or not a given target function is compositional is meaningless; e.g.,
$$f(x)=(x+1)\cosh\left(\log\left(\frac{2+\sqrt{3-2x-x^2}}{x+1}\right)\right)\equiv 2, \qquad x\in [0,1].$$
However, the direct and converse theorems show that if we know in advance that the target function is not as smooth as the degree of approximation by the networks indicates, then the blessing of compositionality must be playing some role.


\end{document}